\documentclass{article}

\usepackage[utf8]{inputenc}

\PassOptionsToPackage{numbers,sort&compress}{natbib}

\usepackage[preprint]{neurips_2024}

\usepackage{amsmath,amssymb,mathtools}
\usepackage{amsthm}
\usepackage{microtype}
\usepackage{graphicx}
\graphicspath{{../}}
\usepackage{booktabs}
\usepackage{xcolor}
\usepackage{booktabs,tabularx,array}
\newcolumntype{L}[1]{>{\raggedright\arraybackslash}p{#1}}

\usepackage{hyperref}
\usepackage[capitalise,noabbrev]{cleveref}

\theoremstyle{plain}
\newtheorem{proposition}{Proposition}
\theoremstyle{definition}

\definecolor{taskgreen}{HTML}{00D7A1}
\definecolor{taskorange}{HTML}{FF8F00}
\definecolor{taskred}{HTML}{EF0037}

\newcommand{\gdot}{\textcolor{taskgreen}{$\bullet$}}
\newcommand{\mdot}{\textcolor{taskorange}{$\bullet$}}
\newcommand{\rdot}{\textcolor{taskred}{$\bullet$}}

\title{Foundation Model Forecasts: Form and Function}

\author{%
  Alvaro Perez-Diaz\thanks{Corresponding author: \protect\raisebox{-0.22em}{\includegraphics[height=0.9em]{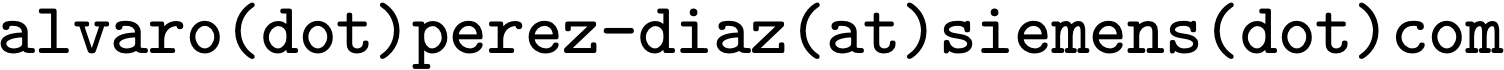}}} \\
  Senseye, Siemens Digital Industries \\
  \And
  James C. Loach \\
  Senseye, Siemens Digital Industries \\
  \AND
  Danielle E. Toutoungi \\
  Senseye, Siemens Digital Industries \\
  \And
  Lee Middleton \\
  Senseye, Siemens Digital Industries \\
}

\begin{document}
\maketitle

\begin{abstract}
Time-series foundation models (TSFMs) achieve strong forecast accuracy, yet accuracy alone does not determine practical value. The form of a forecast---point, quantile, parametric, or trajectory ensemble---fundamentally constrains which operational tasks it can support. We survey recent TSFMs and find that two-thirds produce only point or parametric forecasts, while many operational tasks require trajectory ensembles that preserve temporal dependence. We establish when forecast types can be converted and when they cannot: trajectory ensembles convert to simpler forms via marginalization without additional assumptions, but the reverse requires imposing temporal dependence through copulas or conformal methods. We prove that marginals cannot determine path-dependent event probabilities—infinitely many joint distributions share identical marginals but yield different answers to operational questions. We map six fundamental forecasting tasks to minimal sufficient forecast types and provide a task-aligned evaluation framework. Our analysis clarifies when forecast type, not accuracy, differentiates practical utility.
\end{abstract}

\section{Introduction}

Time-series foundation models (TSFMs) have emerged as a transformative approach to time-series forecasting, achieving strong zero-shot and few-shot performance across diverse datasets and domains \cite{jin_large_2023,liang_foundation_2024,zhang_large_2024,jiang_empowering_2024,su_large_2024,ye_survey_2024}. Recent models demonstrate impressive accuracy on benchmark tasks, matching or exceeding specialized statistical methods and earlier deep-learning architectures. However, forecast accuracy alone does not determine practical utility. The \textit{form} of a forecast—whether it outputs point forecasts, quantile forecasts, parametric forecasts, or trajectory ensembles—fundamentally constrains which downstream applications it can support.

Consider a risk manager evaluating portfolio value-at-risk (VaR) over a five-day horizon. A point forecast of daily returns, no matter how accurate, cannot answer ``\textit{what is the probability our cumulative loss exceeds 5\%?}''. A quantile forecast trained on levels $\{0.1,\ldots,0.9\}$ can interpolate within this range but cannot reliably estimate tail risk at the 1st or 99th percentile without (risky) extrapolation. A parametric forecast does provide a distribution at each time step, but computing the distribution of the five-day sum requires assumptions about temporal dependence. Knowing the distribution of predicted returns for tomorrow and for the day after is not sufficient to determine whether losses will cluster (high correlation) or cancel out (negative correlation)—yet this temporal structure is exactly what determines cumulative risk. Only a trajectory ensemble—sampled paths from the joint predictive distribution—directly yields the answer through Monte Carlo estimation, because it explicitly represents how returns co-vary over time. This fundamental limitation is not a modeling choice but a mathematical fact: per-step marginal distributions cannot uniquely determine probabilities of path-dependent events without additional assumptions about dependence structure (formalized in Proposition~\ref{prop:nonident}). Similar constraints arise across domains: predictive maintenance requires first-passage probabilities, capacity planning requires simultaneous forecast bands, and climate loss modeling requires scenario generation with realistic temporal structure.

\textbf{The gap.} Existing TSFM surveys \cite{jin_large_2023,liang_foundation_2024,zhang_large_2024,jiang_empowering_2024,su_large_2024,ye_survey_2024} comprehensively cover architectural innovations and training strategies. However, they provide limited guidance on which forecast types enable which operational tasks, when conversions between forecast types are valid, and how to select an evaluation metric appropriate for a given task. This gap matters increasingly as TSFMs mature: when multiple models achieve similar accuracy, forecast type becomes the primary differentiator of practical value. Moreover, practitioners may wish to ensemble forecasts from multiple TSFMs to improve robustness—a strategy proven effective across machine learning domains—but must first convert models to a common forecast type, potentially discarding valuable information. Our survey of published TSFMs reveals an important mismatch between model capabilities and application requirements: the majority produce only point forecasts or parametric forecasts, while trajectory ensembles—the most expressive forecast type capable of addressing path-dependent questions like threshold crossing and aggregate risk without post-processing—remain relatively rare. This detailed survey can be found in \Cref{app:survey}.

Two recent studies have examined TSFMs from a lens similar to ours, though with important limitations. \citeauthor{adler_calibration_2025} analyze the calibration properties of several existing TSFMs, but restrict their analysis to marginal calibration at individual future time steps without examining operational tasks. \citeauthor{achour_foundation_2025} explore conformal prediction methods for TSFMs, yet similarly focus only on marginal coverage without considering downstream applications. To our knowledge, no prior work has systematically examined the interplay between forecast types, operational tasks, forecast convertibility, and evaluation metrics.

\textbf{Our contributions.} This paper provides a systematic task-oriented analysis of TSFM forecasts, with formal theory characterizing when different forecast types enable different applications. We make four main contributions:

\begin{enumerate}
\item \textbf{A task-oriented taxonomy.} We formalize four forecast types—point forecasts, quantile forecasts, parametric forecasts, and trajectory ensembles—and characterize their expressiveness hierarchy (\Cref{sec:fore_def}).

\item \textbf{Task-forecast mapping with sufficiency results.} We identify six canonical forecasting problems spanning operational domains and derive the minimal forecast type that suffices for each. We prove when simpler forecast types do not suffice for path-dependent questions without additional assumptions (\Cref{sec:usecases}).

\item \textbf{Convertibility theory with impossibility results.} We formalize conversions between forecast types through a directed graph, proving when transformations require no additional assumptions (trajectory ensembles to any other form via marginalization), when they require structural assumptions (marginals to joint distributions via copulas), and establishing formal impossibility results for path-dependent questions from marginal forecast types (Propositions 1-3, \Cref{sec:convert}).

\item \textbf{Task-aligned evaluation framework.} We provide a systematic mapping of proper scoring rules to forecast types and operational tasks, clarifying when marginal metrics (\emph{e.g.}, CRPS) versus joint metrics (\emph{e.g.}, Energy Score) are required (\Cref{sec:evaluation}).
\end{enumerate}

Our analysis targets both TSFM developers choosing which forecast types to implement and practitioners selecting models for specific application contexts. We show that trajectory ensembles are strictly most expressive—convertible to all other forms through marginalization without additional assumptions—but that simpler forecast types suffice for many tasks. Conversely, attempting to answer path-dependent questions with marginal forecast types requires explicit dependence modeling and validation that is often overlooked in practice.

\textbf{Organization.} Section 2 defines the four forecast types. Section 3 maps six fundamental forecasting tasks to minimally-sufficient forecast types. Section 4 formalizes convertibility between forecast types. Section 5 presents task-aligned evaluation methodology. Section 6 concludes with implications for research and practice.

\section{Form}
\label{sec:fore_def}

We formalize four forecast types that TSFMs can produce. These differ fundamentally in what downstream questions they can answer. This taxonomy is not exhaustive—other forecast types are conceivable—but it encompasses all types observed in published TSFMs to date. For each type, we provide its mathematical definition, key properties, and operational implications. Throughout, we consider univariate forecasting: given history $\mathbf{y}_{1:T}=(y_1,\dots,y_T)$ with $y_t \in \mathbb{R}$, forecast the next $h$ time steps. The analysis extends naturally to multivariate series and to models incorporating time or exogenous covariates. \Cref{fig:tsfm_outputs} illustrates all four forecast types for the same underlying forecast. The full review of published TSFMs and the output type of each can be found in \Cref{app:survey}.

\begin{figure}
  \centering
  \includegraphics[width=0.9\linewidth]{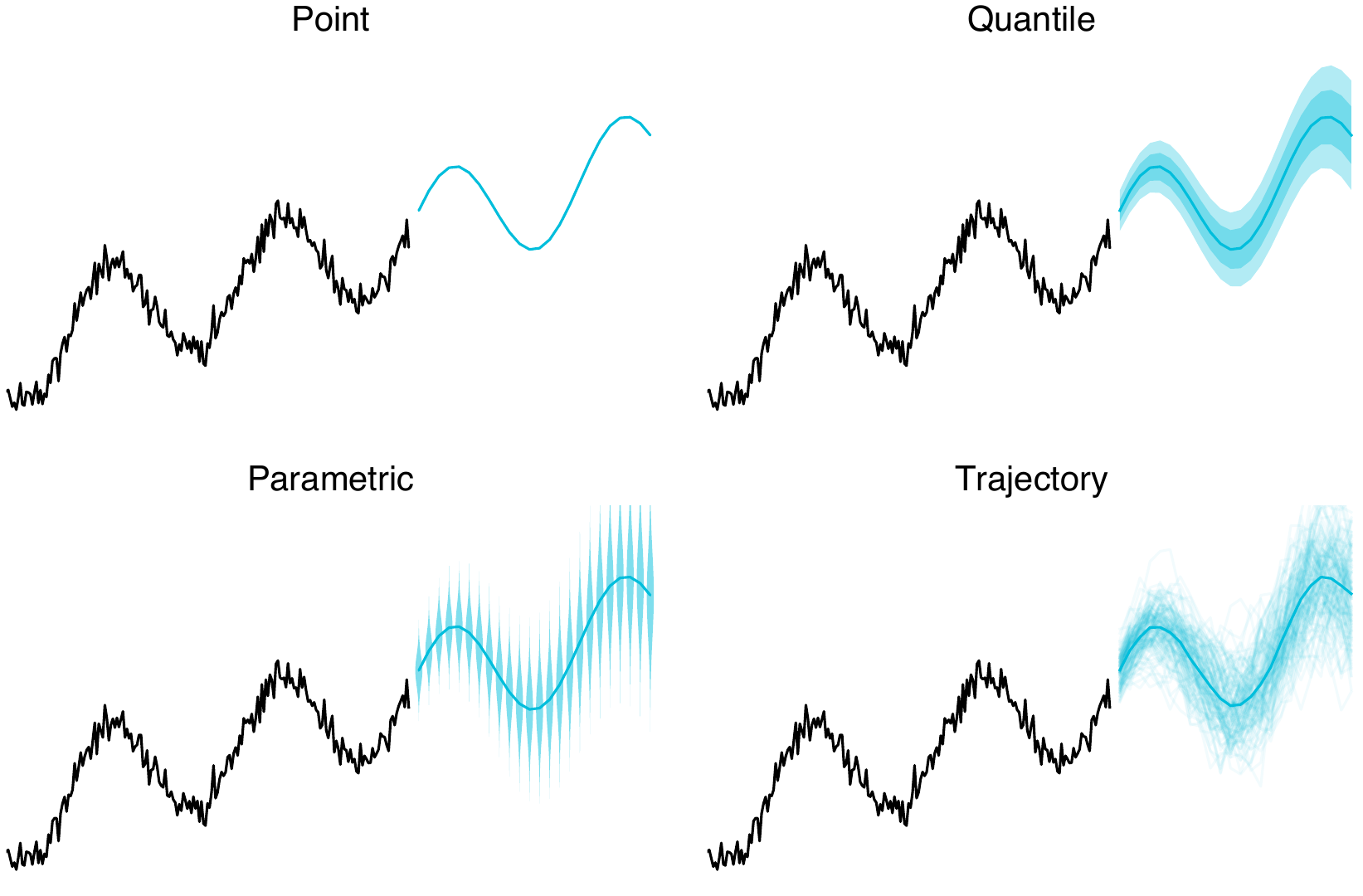}
  \caption{Four forecast types derived from the same underlying predictive distribution. \emph{Point} shows per-step medians or means—note the smooth line is not a sampled trajectory but a sequence of coordinate-wise summaries. \emph{Quantile} and \emph{Parametric} display per-step marginal distributions; uncertainty bands at each time step do not imply joint coverage over the horizon. \emph{Trajectory} shows sampled paths from the joint distribution, preserving temporal dependence.}
  \label{fig:tsfm_outputs}
\end{figure}

\paragraph{Terminology.}
We use \emph{parametric} to mean \emph{per-step parametric marginals} (one predictive CDF per future time step, $F_{T+k}$).
\emph{Trajectory ensemble} means sampled joint paths over the forecast horizon.
We reserve \emph{marginals} for per-step distributions and \emph{joint} for the full path distribution.

\paragraph{Notation.}
We assume that before time $T$ history is known. After $T$, the future is unknown and quantities there are forecasts. We refrain from using mathematical hat notation to avoid clutter, \emph{e.g.} the point forecast $\hat{y}_{T+k}$ will be denoted simply as $y_{T+k}$. In \Cref{sec:evaluation}, where we discuss both forecasts and actual realizations, we introduce the hats for clarity.

\subsection{Point forecasts}

A point forecast outputs a deterministic value per future time step: $\mathbf{y}_{T+1:T+h}=(y_{T+1},\dots,y_{T+h})$ where $y_{T+k} \in \mathbb{R}$. This representation provides no native uncertainty quantification. Practitioners must add external calibration through post-hoc methods. Common approaches include split conformal prediction, which provides coverage guarantees by computing residual quantiles on calibration data, and residual bootstrapping, which adds realistic noise patterns from historical forecast errors \cite{kuleshov_accurate_2018,angelopoulos_gentle_2021,stankeviciute_conformal_2021,xu_conformal_2020,xu_conformal_2024,achour_foundation_2025}.

However, conformal methods guarantee \emph{coverage} (the fraction of realizations within intervals) but not calibrated event probabilities. They answer ``\textit{where will 95\% of realizations fall?}'' not ``\textit{what is $\Pr(\text{event})$?}''. Bootstrap ensembles inherit the residual model's assumptions and may miss tail behavior or regime changes \cite{adler_calibration_2025}.

A multi-step point forecast $\mathbf{y}_{T+1:T+h}$ is a vector of \emph{per-step functionals}, typically the mean $y_{T+k} = \mathbb{E}[Y_{T+k} \mid \mathbf{y}_{1:T}]$ if trained with squared error, or the median $y_{T+k} = \text{median}(Y_{T+k} \mid \mathbf{y}_{1:T})$ if trained with absolute error. Connecting these values yields a convenient visualization, but the connected line is just a sequence of per-step summaries, not a draw from the joint predictive distribution $p(\mathbf{y}_{T+1:T+h} \mid \mathbf{y}_{1:T})$. This distinction has an immediate visual consequence: point forecasts appear systematically smoother than individual realizations. When stochastic trajectories exhibit realistic high-frequency variability, coordinate-wise averaging (mean/median) attenuates high-frequency variation while preserving low-frequency components, producing a smooth central tendency that no single trajectory follows. This smoothness is mathematically correct, not an artifact.

For applications requiring plausible scenarios—stress-testing operational systems, simulating decision rules, or generating realistic what-if paths—the point forecast provides an answer to the wrong question. It shows expected central tendency, not a forecasted realization. Path-dependent quantities like threshold crossings, run lengths, and aggregate sums computed from point forecasts lack probabilistic meaning because the underlying joint distribution is absent.

\subsection{Quantile forecasts}

Given a fixed set of $L$ quantile levels in ascending order, $\mathcal{Q}=\{q_\ell\}_{\ell=1}^L\subset(0,1)$ (\emph{e.g.}, $\mathcal{Q}=\{0.1, 0.2, \ldots, 0.9\}$), the TSFM produces quantile values $\{Q_{T+k}(q_\ell)\}_{\ell=1}^L$ for each future step $k=1,\ldots,h$. These define marginal distributions at each time step. Quantile forecasts enable probabilistic statements about individual time steps without assuming a specific distributional family. They are robust to misspecification in the tails if the training quantile levels span the region of interest, and interpolation between predicted quantiles is straightforward via monotone splines or linear methods.

Nevertheless, quantile forecasts are constrained to the levels used during training. Extrapolating to extreme quantiles (\emph{e.g.}, 1st or 99th percentile for tail risk) when trained on $\{0.1,\ldots,0.9\}$ is unreliable without retraining. More fundamentally, these marginals do not identify a joint distribution over the horizon—answering path-dependent questions requires additional assumptions about temporal dependence (see \Cref{sec:convert}).

\subsection{Parametric forecasts (per-step marginals)}

Parametric forecasts produce distribution parameters for each future time step, defining marginal CDFs $\{F_{T+k}\}_{k=1}^h$. For a parametric family with parameter vector $\boldsymbol{\theta} \in \Theta$, the TSFM produces $(\boldsymbol{\theta}_{T+1}, \ldots, \boldsymbol{\theta}_{T+h})$. For example, Gaussian forecasts yield $\boldsymbol{\theta}_{T+k} = (\mu_{T+k}, \sigma_{T+k}^2)$ at each step. The choice of distributional family is crucial and decided at train time. This representation enables analytical calculations when the distributional family supports them—moments, tail probabilities, and interval inversion are often available in closed form. Common families include Gaussian, Student-$t$, negative binomial, and their mixtures.

When the chosen parametric family correctly matches the data-generating process, parametric forecasts achieve the Cramér-Rao lower bound asymptotically for parameter estimation; this can translate to sharper predictions in some settings, but does not by itself guarantee optimal prediction-interval width \cite{lehmann_theory_1998,a_w_van_der_vaart_asymptotic_1998}. However, this advantage disappears under misspecification. When the true distribution differs from the assumed family—for instance, heavy-tailed or skewed data modeled as Gaussian—parameter estimates become biased and prediction intervals lose their calibration guarantees. Since TSFMs are designed to forecast arbitrary time series without domain-specific knowledge, committing to a single parametric family across all data seems difficult to justify. Moreover, as with quantiles, these marginals do not identify a joint distribution; practitioners often assume independence or add dependence post-hoc via copulas or rank reordering (see \Cref{sec:convert}).

\subsection{Trajectory ensembles (joint paths)}

Trajectory ensembles produce sampled paths from the joint predictive distribution: $(\mathbf{y}_{T+1:T+h}^{(1)}, \ldots, \mathbf{y}_{T+1:T+h}^{(M)})$ where each $\mathbf{y}_{T+1:T+h}^{(m)}$ is a complete future realization and $M$ is the ensemble size. These trajectories preserve temporal dependence: $p(\mathbf{y}_{T+1:T+h}|\mathbf{y}_{1:T}) = \prod_{k=1}^{h} p(y_{T+k}|\mathbf{y}_{1:T+k-1})$. Generation methods include autoregressive sampling \cite{ansari_chronos_2024}, diffusion models \cite{cao_timedit_2025,yuan_diffusion-ts_2024,tashiro_csdi_2021}, and latent variable approaches \cite{hu_swinvrnn_2023,kucinski_tsgt_2024}.

Trajectory ensembles are strictly most expressive. They can be converted to quantile forecasts via empirical order statistics, to parametric forecasts via empirical CDFs, and to point forecasts via means or medians (see \Cref{sec:convert}). Because they represent the joint distribution, they directly support path-dependent questions: first-passage times, threshold crossings, exceedance durations, and cumulative targets all require joint future paths, not just per-step marginals \cite{redner_guide_2001}. 

\subsection{The expressiveness hierarchy}

Point forecasts provide no native uncertainty quantification. Among probabilistic forecast types, quantile forecasts and parametric forecasts provide \emph{marginal} uncertainty at each time step, while trajectory ensembles provide a \emph{joint} distribution over the full horizon. This hierarchy determines which tasks each forecast type can support, as detailed in \Cref{sec:usecases}. Critically, the reverse conversions—from marginals to joints—generally lose temporal dependence information or require explicit assumptions about dependence structure (copulas, reordering schemes). For path-dependent questions, such conversions can be invalid without correctly modeling temporal dependence. We formalize these convertibility relationships in \Cref{sec:convert}.

\section{Function}
\label{sec:usecases}

We identify six canonical forecasting tasks that span operational domains. These are very broad and cover a large variety of use cases. For each, we define the problem, formalize the mathematical objective, and determine the minimal forecast type that suffices. \Cref{tab:use_cases} summarizes the complete mapping; subsections below provide the essential formulations. Extended mathematical derivations and detailed examples appear in \Cref{app:b}.

\begin{table}
\centering
\small
\renewcommand{\arraystretch}{1.3}
\begin{tabular}{@{}lcccc@{}}
\toprule
& \multicolumn{4}{c}{\textbf{Forecast type}} \\
\cmidrule(l){2-5}
\textbf{Forecasting task} & \textbf{Point} & \textbf{Quantile} & \textbf{Parametric} & \textbf{Trajectory} \\
\midrule
Pointwise intervals & \mdot & \gdot & \gdot & \gdot \\
Pathwise bands & \mdot & \mdot & \mdot & \gdot \\
Event probabilities & \rdot & \mdot & \mdot & \gdot \\
Threshold crossing & \rdot & \mdot & \mdot & \gdot \\
Window aggregates & \mdot & \mdot & \mdot & \gdot \\
Scenario generation & \rdot & \mdot & \mdot & \gdot \\
\bottomrule
\end{tabular}
\vspace{0.75em}
\caption{Task-forecast compatibility matrix for TSFM forecasts. Symbols: \gdot~(sufficient—natively supports task), \mdot~(feasible with assumptions—requires post-processing, copulas, or conformal methods), \rdot~(unsuitable—cannot reliably support task). Point forecasts need conformal prediction for uncertainty quantification but cannot estimate probabilities. Quantile and parametric forecasts require copulas to reconstruct joint distributions for path-dependent tasks. Trajectory ensembles natively support all tasks.}
\label{tab:use_cases}
\end{table}

\subsection{Pointwise prediction intervals}
\label{sec:use_case_point_uncertainty}

Many operational applications require per-step uncertainty: ``\textit{what will the value be at each future time and what is my confidence?}''. A manufacturer monitoring furnace temperature for quality control needs minute-level prediction intervals (PIs) to trigger alarms. A telecom operator uses per-interval traffic PIs to schedule autoscaling. Both applications are local in time—marginal uncertainty at each step suffices.

Mathematically, given history $\mathbf{y}_{1:T}$ and horizon $h$, a TSFM induces marginals $p_{T+k}(y) = p(y_{T+k} \mid \mathbf{y}_{1:T})$ with CDF $F_{T+k}$. For confidence level $1-\alpha$, the target is the interval $[L_{T+k},U_{T+k}] = [F_{T+k}^{-1}(\alpha/2), F_{T+k}^{-1}(1-\alpha/2)]$ at each step $k=1,\ldots,h$.

Each forecast type supports this task through different methods:

\begin{center}
\small
\begin{tabular}{@{}lp{9.5cm}@{}}
\toprule
\textbf{Forecast type} & \textbf{Method} \\
\midrule
Point \mdot & Post-hoc via split conformal prediction: compute $(1-\alpha)$-quantile of absolute residuals per future time step and add/subtract from forecast (formalized in \Cref{sec:point_convert}). \\
\addlinespace
Quantile \gdot & Interpolate predicted quantiles to levels $\alpha/2, 1-\alpha/2$. Extrapolation to tails requires retraining. \\
\addlinespace
Parametric \gdot & Invert $F_{T+k}$ directly from per-step parameters. \\
\addlinespace
Trajectory \gdot & Compute empirical quantiles from ensemble $\{y^{(m)}_{T+k}\}_m$. \\
\bottomrule
\end{tabular}
\end{center}

\subsection{Pathwise forecast bands}
\label{sec:use_case_path_uncertainty}

Other applications require simultaneous coverage over the entire horizon: ``\textit{give me a band containing the next 24 hours with 95\% confidence}''. Cloud engineers planning maintenance need a 24-h CPU band that holds jointly across all time steps. Dam managers scheduling releases need a 7-day discharge band. Unlike pointwise intervals, these require pathwise guarantees where temporal dependence matters.

Mathematically, for confidence $1-\alpha$, the target is simultaneous coverage: $\Pr(y_{T+1} \in [L_1,U_1], \ldots, y_{T+h} \in [L_h,U_h] \mid \mathbf{y}_{1:T}) \ge 1-\alpha$. Each forecast type supports this task through different methods:

\begin{center}
\small
\begin{tabular}{@{}lp{9.5cm}@{}}
\toprule
\textbf{Forecast type} & \textbf{Method} \\
\midrule
Point \mdot & Pathwise split conformal with sup-norm scores (formalized in \Cref{sec:point_convert}). \\
\addlinespace
Quantile \mdot & Pathwise conformal around predicted quantile curves, or reconstruct CDFs (\Cref{sec:quant2dens}) and sample via copula (\Cref{sec:dens2paths}), or Šidák/Bonferroni adjustment \cite{sidak_rectangular_1967}. \\
\addlinespace
Parametric \mdot & Generate trajectory ensemble via copula-based reconstruction (\Cref{sec:dens2paths}), then apply trajectory method; or use conservative Šidák/Bonferroni adjustment \cite{sidak_rectangular_1967}. \\
\addlinespace
Trajectory \gdot & For each path $m$, compute normalized deviation $d^{(m)} = \max_k |y^{(m)}_{T+k} - m_k|/s_k$ where $m_k$ and $s_k$ are the per-step ensemble center and scale (\emph{e.g.}, median and median absolute deviation, or mean and standard deviation). Set $c$ as $(1-\alpha)$-quantile of $\{d^{(m)}\}$ and band $[m_k - cs_k, m_k + cs_k]$. \\
\bottomrule
\end{tabular}
\end{center}

Bands constructed from a forecast ensemble (\emph{e.g.}, using the empirical $1{-}\alpha$ envelope) correctly reflect the model's internal uncertainty (\emph{i.e.} are \emph{credible} under the model) but do not guarantee $1{-}\alpha$ coverage in finite real-world data unless the model is perfectly specified. Since TSFMs are inevitably misspecified for real-world data, empirical coverage should be validated on held-out sets. For formal frequentist coverage guarantees, apply pathwise split conformal calibration on held-out windows to adjust band width while preserving time dependence (see \Cref{sec:point_convert}).

\subsection{Event and tail probabilities}
\label{sec:use_case_event}

Risk management requires probabilities of events defined over the full horizon: ``\textit{what is the probability our cumulative loss exceeds a threshold?}''. Financial risk managers compute weekly value-at-risk (VaR)—the tail quantile of portfolio loss over five trading days. Logistics planners estimate the probability that total weekly inbound volume exceeds warehouse capacity. These functionals (sums, maxima, \textit{etc.}) demand a coherent joint forecast.

Mathematically, define a functional $g(\mathbf{y}_{T+1:T+h})$ and event $\{g > c\}$. The target is $\Pr(g(\mathbf{y}_{T+1:T+h}) > c \mid \mathbf{y}_{1:T})$. Each forecast type supports this task through different methods:

\begin{center}
\small
\begin{tabular}{@{}lp{9.5cm}@{}}
\toprule
\textbf{Forecast type} & \textbf{Method} \\
\midrule
Point \rdot & Unsuitable—conformal methods yield coverage bands, not calibrated event probabilities \citep{angelopoulos_gentle_2021}. \\
\addlinespace
Quantile \mdot & Reconstruct CDFs (\Cref{sec:quant2dens}), add copula (\Cref{sec:dens2paths}), estimate by Monte Carlo. \\
\addlinespace
Parametric \mdot & Add temporal dependence via copula (\Cref{sec:dens2paths}), simulate paths, estimate by Monte Carlo; some analytical shortcuts exist (\emph{e.g.}, sums of Gaussians). \\
\addlinespace
Trajectory \gdot & Direct Monte Carlo: $\widehat{\Pr}(g > c) = \frac{1}{M}\sum_m \mathbf{1}\{g(\mathbf{y}^{(m)}) > c\}$. \\
\bottomrule
\end{tabular}
\end{center}

VaR is a canonical example: for portfolio returns over five days, compute the tail quantile of cumulative loss via marginal CDFs with assumed dependence (parametric/quantile outputs) or empirical quantiles directly (trajectories). See \Cref{app:var} for a complete formulation.

\subsection{Threshold crossing and persistence}
\label{sec:use_case_threshold}

Predictive maintenance and environmental monitoring require first-passage and run-length distributions: ``\textit{when will vibration first exceed my machine's safety limit?}'' or ``\textit{how long will pollution stay above the legal threshold?}''.

Mathematically, given threshold $C$ and crossing direction $\triangleright \in \{\ge, \le\}$, define first hitting time $\tau = \inf\{k \in \{1,\ldots,h\}: y_{T+k} \triangleright C\}$ with censoring at $h$ if no crossing occurs. The target is the survival function $S(k) = \Pr(\tau > k \mid \mathbf{y}_{1:T})$. Each forecast type supports this task through different methods:

\begin{center}
\small
\begin{tabular}{@{}lp{9.5cm}@{}}
\toprule
\textbf{Forecast type} & \textbf{Method} \\
\midrule
Point \rdot & Only point estimate $\hat{\tau} = \min\{k: y_{T+k} \triangleright C\}$; no uncertainty quantification. \\
\addlinespace
Quantile \mdot & Reconstruct CDFs (\Cref{sec:quant2dens}), then apply parametric method; inherits independence limitation. \\
\addlinespace
Parametric \mdot & Compute per-step crossing probabilities from marginal CDFs. Under independence: $S(k) = \prod_{j=1}^k (1-p_j)$ where $p_j = \Pr(y_{T+j} \triangleright C)$. This ignores path history and underestimates persistence. \\
\addlinespace
Trajectory \gdot & Direct empirical estimate: $$\widehat{\Pr}(\tau=k) = \frac{1}{M}\sum_m \mathbf{1}\{\text{first crossing at }k\text{ in path }m\}$$ \\
\bottomrule
\end{tabular}
\end{center}

The independence approximation in the parametric method typically underestimates persistence in positively autocorrelated series: runs of high (or low) values make crossings cluster in time, violating the memoryless assumption. Trajectory ensembles preserve autocorrelation structure and yield correct run-length distributions.

Complete survival analysis formulation, including hazard functions and worked examples for vibration monitoring and pollution persistence, appears in \Cref{app:survival}.

\subsection{Window aggregates}
\label{sec:use_case_aggregates}

Planning tasks often require totals over a time window: ``\textit{what is the expected demand during the holiday week?}''. Retailers setting inventory and staffing need the distribution of total demand over a 7-day promotional period. Public health agencies allocating resources need expected influenza cases over the next 14 days.

Mathematically, let $W \subseteq \{1,\ldots,h\}$ denote the window of interest---within the forecast horizon $h$. The target is the aggregate $Z = \sum_{k \in W} y_{T+k}$, either its expectation $\mathbb{E}[Z \mid \mathbf{y}_{1:T}]$ or full distribution. Each forecast type supports this task through different methods:

\begin{center}
\small
\begin{tabular}{@{}lp{9.5cm}@{}}
\toprule
\textbf{Forecast type} & \textbf{Method} \\
\midrule
Point \rdot & Sum forecasts $\hat{Z} = \sum_{k \in W} y_{T+k}$; provides baseline but no uncertainty. \\
\addlinespace
Quantile \mdot & Approximate per-step means by integrating quantile functions, then sum. For full distribution, sample uniformly and sum. Note: quantile of sum $\neq$ sum of quantiles.\\
\addlinespace
Parametric \mdot & Under time independence, mean: $\mathbb{E}[Z] = \sum_{k \in W} \mathbb{E}[Y_{T+k}]$ from marginal parameters. Full distribution: analytically available for some families under independence (\emph{e.g.}, sums of Gaussians); otherwise, add dependence via copula (\Cref{sec:dens2paths}) and simulate. \\
\addlinespace
Trajectory \gdot & For each path, compute $Z^{(m)} = \sum_{k \in W} y^{(m)}_{T+k}$. Then $\widehat{\mathbb{E}}[Z] = \frac{1}{M}\sum_m Z^{(m)}$ and empirical CDF $\widehat{F}_Z(z) = \frac{1}{M}\sum_m \mathbf{1}\{Z^{(m)} \le z\}$; preserves time dependence. \\
\bottomrule
\end{tabular}
\end{center}

\subsection{Scenario generation and ranking}
\label{sec:use_case_scenarios}

Strategic planning requires concrete future scenarios to stress-test decisions: ``\textit{generate plausible metro load paths to test capacity limits}'' or ``\textit{rank climate-loss scenarios by probability and severity}''. Transportation operators simulate multiple load trajectories to identify bottlenecks. Insurers generate and rank weather scenarios to set reserves and premiums.

Mathematically, the target is to sample scenarios $\{\mathbf{y}^{(m)}_{T+1:T+h}\}_{m=1}^M \sim p(\mathbf{y}_{T+1:T+h} \mid \mathbf{y}_{1:T})$, compute functionals $g_j(\mathbf{y}^{(m)})$ (\emph{e.g.}, peak load, exceedance duration, cumulative demand), and rank by probability or severity. Each forecast type supports this task through different methods:

\begin{center}
\small
\begin{tabular}{@{}lp{9.5cm}@{}}
\toprule
\textbf{Forecast type} & \textbf{Method} \\
\midrule
Point \rdot & Single baseline $\mathbf{y}_{T+1:T+h}$; residual bootstrap gives rough variability but is uncalibrated. \\
\addlinespace
Quantile \mdot & Reconstruct CDFs (\Cref{sec:quant2dens}), add copula (\Cref{sec:dens2paths}), sample paths; quality depends on copula choice. \\
\addlinespace
Parametric \mdot & Independent per-step sampling breaks temporal coherence; use copula-based reconstruction (\Cref{sec:dens2paths}) for better paths; quality depends on copula choice. \\
\addlinespace
Trajectory \gdot & Scenarios available directly. Compute functionals: peak $P^{(m)} = \max_k y^{(m)}_{T+k}$, exceedance $E^{(m)} = \sum_k \mathbf{1}\{y^{(m)}_{T+k} > C\}$, cumulative $Z^{(m)} = \sum_k y^{(m)}_{T+k}$. Rank by empirical CDFs or cluster and assign weights. \\
\bottomrule
\end{tabular}
\end{center}

Detailed scenario ranking procedures for insurance applications, including clustering, severity metrics, and exceedance curves, appear in \Cref{app:insurance}.

\section{Conversion}
\label{sec:convert}

Section~\ref{sec:usecases} demonstrated that different operational tasks require different forecast types. In practice, practitioners often need to convert between forecast types for two primary reasons: (1) adapting an existing TSFM to new tasks, and (2) ensembling forecasts from multiple models. Ensemble methods—combining predictions from multiple models—are a commonplace in modern machine learning. In forecasting, ensembles aggregate information from models with different architectures, training data, or inductive biases, often producing more robust predictions than any single model \cite{timmermann_forecast_2006}. However, when constituent models produce different forecast types (\emph{e.g.}, one outputs trajectory ensembles, another outputs quantiles), practitioners must convert to a common representation before combination. This conversion can preserve or discard critical information depending on the target forecast type and the conversion method. We formalize when such conversions are valid, what information they preserve or discard, and which conversions are impossible without additional assumptions.

\Cref{fig:overview} summarizes the convertibility landscape as a directed graph. Green arrows indicate conversions that require no additional assumptions beyond the forecast itself (marginalization from trajectory ensembles); orange arrows mark conversions that require structural assumptions (copulas, distributional families) to reconstruct information not present in the source forecast type. We establish three theoretical results characterizing this structure, then provide practical conversion methods for each forecast type.

\begin{figure}[t]
\centering
\includegraphics[width=0.6\textwidth]{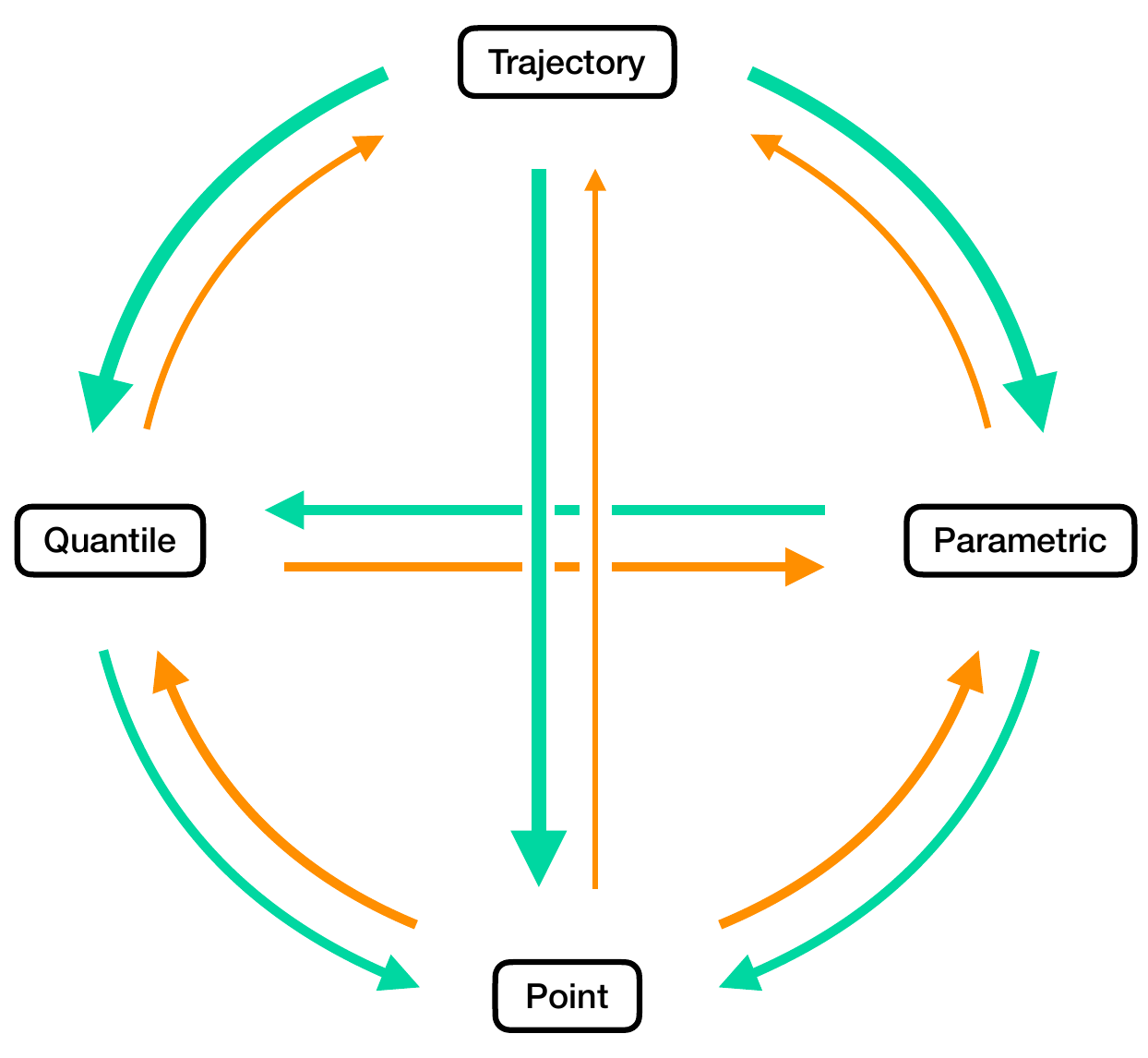}
\caption{Convertibility between forecast types. Trajectory ensembles (top) are strictly most expressive, enabling direct conversions (green arrows) to parametric forecasts, quantiles, and points through marginalization—no additional assumptions required, though temporal dependence information is discarded. Conversions to richer forms (orange arrows) require structural assumptions: copulas for temporal dependence (parametric/quantile → trajectory), distributional families (quantile → parametric), or conformal prediction (point → probabilistic outputs). Lateral conversions between parametric and quantile forecasts preserve marginal information but differ in form. All conversions are detailed in Section~\ref{sec:convert} and summarized in Table~\ref{tab:conversion_summary}.}
\label{fig:overview}
\end{figure}

\subsection{Theoretical foundations}

Three propositions characterize the convertibility structure. Proposition~\ref{prop:expressiveness} establishes the expressiveness hierarchy; Proposition~\ref{prop:nonident} proves that marginals cannot identify path-dependent probabilities; Proposition~\ref{prop:quantiles-to-parametric} determines when quantiles uniquely recover parametric forms.

\begin{proposition}[Expressiveness hierarchy]\label{prop:expressiveness}
Trajectory ensembles are strictly most expressive. They determine parametric forecasts, quantile forecasts, and point forecasts via marginalization without requiring additional assumptions and modeling beyond the ensemble itself.
\end{proposition}

\begin{proof}
Let $\{\mathbf{y}^{(m)}_{T+1:T+h}\}_{m=1}^M$ denote trajectory samples. For any step $k$, the empirical CDF $\widehat{F}_{T+k}(x) = \frac{1}{M}\sum_{m=1}^M \mathbf{1}\{y^{(m)}_{T+k} \le x\}$ converges uniformly to the true marginal $F_{T+k}$ by the Glivenko-Cantelli theorem \cite{a_w_van_der_vaart_weak_2013}. Quantiles follow by inversion and points are summaries (mean, median, \emph{etc.}). Thus trajectories determine all simpler forms through direct marginalization, requiring no distributional assumptions or external modeling. However, marginalization discards the joint distribution structure—temporal dependence information is irreversibly lost.

Strictness follows from Sklar's theorem \cite{sklar_random_1973}: any joint distribution can be written as $F(y_1,\ldots,y_h) = C(F_{T+1}(y_1),\ldots,F_{T+h}(y_h))$ where $C$ is a copula and $F_{T+k}$ are the marginals. For $h \ge 2$, infinitely many copulas $C$ yield the same marginals $\{F_{T+k}\}$ but different joint distributions. For example, the independence copula $C^\perp(u_1,\ldots,u_h) = \prod_k u_k$ and the comonotonic copula $C^+(u_1,\ldots,u_h) = \min(u_1,\ldots,u_h)$ preserve identical marginals but produce opposite temporal dependence structures, yielding different probabilities for path-dependent events \cite{nelsen_introduction_2006}. Thus marginals cannot uniquely determine the joint distribution: trajectory ensembles contain strictly more information.
\end{proof}

\begin{proposition}[Non-identifiability of path-dependent questions from marginals]\label{prop:nonident}
Fix $h \ge 2$ and continuous marginals $F_1,\ldots,F_h$. There exist at least two distinct joint distributions with these marginals that yield different probabilities for path-dependent events.
\end{proposition}

\begin{proof}
By Sklar's theorem, any joint distribution with marginals $F_1, F_2$ can be written as $F(y_1,y_2) = C(F_1(y_1), F_2(y_2))$ for some copula $C$ \cite{sklar_random_1973}. Since the copula is not determined by the marginals alone, infinitely many joint distributions share the same marginals but differ in their dependence structure.

The Fréchet-Hoeffding bounds provide extremal examples: the comonotonic copula $C^+(u_1,u_2) = \min(u_1,u_2)$ induces perfect positive dependence (when one variable is high, the other tends to be high), while the countermonotonic copula $C^-(u_1,u_2) = \max(u_1+u_2-1,0)$ induces perfect negative dependence (when one variable is high, the other tends to be low). Both preserve arbitrary marginals $F_1, F_2$ but produce opposite temporal dependence \cite{nelsen_introduction_2006}.

Consider any path-dependent event defined by a functional $g(Y_1,Y_2)$ that depends non-trivially on the joint behavior—for instance, sums, products, maxima, or threshold crossings. Under comonotonic coupling, extreme values of $Y_1$ and $Y_2$ occur together, concentrating probability mass along the diagonal. Under countermonotonic coupling, extreme values occur in opposition, spreading probability mass differently. These distinct dependence structures yield different probabilities for path-dependent events, even though the marginals remain identical. The argument extends to $h > 2$ via vine copula decompositions \cite{aas_pair-copula_2009}.
\end{proof}

\begin{proposition}[Quantiles to parametric: local identifiability]\label{prop:quantiles-to-parametric}
A finite set of quantiles constrains but does not uniquely determine a parametric distribution—infinitely many distributions can match the same quantile values. However, if we fix a parametric family with $d$ parameters, then $L \ge d$ well-chosen quantile levels can locally identify the parameters uniquely.

Formally: Let $\{F(\cdot \mid \theta): \theta \in \Theta \subset \mathbb{R}^d\}$ be an identifiable parametric family, and define the quantile mapping
\[
\Phi(\theta) = (Q_\theta(q_1),\ldots,Q_\theta(q_L))
\]
where $Q_\theta(q)$ is the $q$-quantile of $F(\cdot \mid \theta)$ and $\{q_\ell\}_{\ell=1}^L \subset (0,1)$. If the Jacobian $J(\theta) = \partial\Phi(\theta)/\partial\theta^\top$ has rank $d$ near $\theta_\star$ and $L \ge d$, then $\Phi$ is locally injective at $\theta_\star$—at most one parameter vector matches all $L$ quantiles.
\end{proposition}

\begin{proof}
A full-rank Jacobian $J(\theta_\star)$ implies local injectivity via the implicit function theorem \cite{lehmann_theory_1998,steven_g_krantz_implicit_2012}. If $L < d$, the system is underdetermined. If chosen levels make $J(\theta)$ rank-deficient (\emph{e.g.}, median-only for scale parameters), multiple parameter vectors can match the quantiles.

We can expand for some common distributions. For a Gaussian $N(\mu,\sigma^2)$ with $d=2$ unknown parameters, any two distinct levels $q_1 \neq q_2$ uniquely determine $(\mu,\sigma)$ via $Q(q_\ell) = \mu + \sigma z_{q_\ell}$ where $z_q$ is the standard normal quantile. For Student-$t$ with $(\mu,\sigma,\nu)$ and $d=3$, at least three well-spaced levels including tails are needed; median-centered levels convey little information about $\nu$. For finite mixtures, more than $d$ quantiles may be needed for numerical stability \cite{teicher_identifiability_1963}.
\end{proof}

\subsection{From trajectory ensembles (marginalization)}

Trajectory ensembles $\{\mathbf{y}^{(m)}_{T+1:T+h}\}_{m=1}^M$ provide the most flexible starting point. Conversions to simpler forecast types proceed via marginalization—computing per-step statistics from the ensemble without requiring additional assumptions. These conversions accurately recover marginal distributions asymptotically but irreversibly discard temporal dependence information.

The primary limitation is computational cost: storage, communication, and inference scale with ensemble size $M$. Extreme quantiles, often used for risk estimation (\Cref{sec:use_case_event}), require large $M$ for stable estimation, though not all tasks demand large ensembles and $M$ can be adjusted at inference time. This trade-off between expressiveness and computational efficiency is illustrated by the evolution from Chronos \cite{ansari_chronos_2024}, a trajectory ensemble model, to ChronosBolt \cite{ansari_chronosbolt_2024} and Chronos-2 \cite{ansari_chronos-2_2025}, which produce quantile forecasts. The authors report improvements in accuracy and much faster inference with their quantile models, demonstrating that simpler forecast types offer practical advantages when applications require only marginal uncertainty. However, this efficiency gain comes at the cost of flexibility: ChronosBolt and Chronos-2 need external modeling and assumptions for answering some operational questions, as discussed in Sections \ref{sec:usecases} and \ref{sec:convert}, whereas Chronos supports such tasks natively.

\subsubsection{To quantile forecasts}

At each step $k$, compute the empirical CDF from ensemble values $\{y^{(m)}_{T+k}\}_{m=1}^M$:
\[
\widehat{F}_{T+k}(y) = \frac{1}{M}\sum_{m=1}^M \mathbf{1}\{y^{(m)}_{T+k} \le y\}, \quad \widehat{Q}_{T+k}(q) = \inf\{y: \widehat{F}_{T+k}(y) \ge q\}.
\]
In practice, smooth quantile curves over $k$ and apply monotone rearrangement to ensure $\widehat{Q}_{T+k}(q_1) \le \widehat{Q}_{T+k}(q_2)$ for $q_1 < q_2$.

Information loss: Marginalization preserves per-step distributional information but discards the joint distribution structure. Path-dependent questions (first-passage, run lengths) cannot be answered from quantiles alone without reconstructing temporal dependence via additional assumptions. Extreme quantiles require large $M$ for stable estimation.

\subsubsection{To parametric forecasts}

Fit a parametric family to the empirical distribution at each step. For family $F(\cdot \mid \theta)$, estimate $\theta_{T+k}$ by maximum likelihood:
\[
\hat{\theta}_{T+k} = \arg\max_\theta \sum_{m=1}^M \log f(y^{(m)}_{T+k} \mid \theta)
\]
or use method-of-moments, kernel density estimation, or empirical CDFs directly.

Information loss: Like quantile conversion, parametric fitting discards temporal dependence. Additionally, the parametric fit may distort tail behavior if the chosen family is misspecified (\emph{e.g.}, Gaussian tails for heavy-tailed data).

\subsubsection{To point forecasts}

Summarize each step's ensemble with a point statistic: mean $\hat{y}_{T+k} = \frac{1}{M}\sum_m y^{(m)}_{T+k}$, median, \emph{etc}. Choose the summary aligned with the decision loss: mean for squared error, median for absolute error.

Information loss: Complete—both distributional uncertainty and temporal dependence are discarded. Point forecasts cannot support any probabilistic task without external uncertainty quantification.

\subsection{From parametric forecasts}

Given marginal CDFs $\{F_{T+k}\}_{k=1}^h$ parametrized by $\{\boldsymbol{\theta}_{T+k}\}$, conversions to simpler forms are straightforward, but generating trajectories requires specifying temporal dependence.

\subsubsection{To quantile forecasts}

Invert each marginal CDF: $Q_{T+k}(q) = F_{T+k}^{-1}(q)$. Many families have closed-form inverses (Gaussian: $\Phi^{-1}$; exponential: $-\log(1-q)/\lambda$). For others, use numerical root-finding on $F_{T+k}(x) = q$.

Caveats: Quantile quality depends on the parametric family. Misspecified tails yield biased extreme quantiles. No temporal dependence information is added—this is a lateral conversion between marginal representations.

\subsubsection{To point forecasts}

Extract the desired summary statistic from each marginal: mean $\mu_{T+k} = \mathbb{E}[Y_{T+k}]$, median $Q_{T+k}(0.5)$, or mode $\arg\max_y f_{T+k}(y)$. For mixtures, compute moments analytically or via numerical integration.

Information loss: Only the chosen statistic is retained; distributional shape, spread, and tails are discarded.

\subsubsection{To trajectory ensembles}
\label{sec:dens2paths}

Marginals $\{F_{T+k}\}$ do not determine a joint distribution. By Sklar's theorem \cite{sklar_random_1973,nelsen_introduction_2006}, any joint can be written as
\[
\Pr(Y_{T+1} \le y_1, \ldots, Y_{T+h} \le y_h) = C(F_1(y_1),\ldots,F_h(y_h))
\]
where $C$ is a copula on $(0,1)^h$. To sample trajectories, specify $C$, draw $U^{(m)} \sim C$, and transform: $y^{(m)}_{T+k} = F_{T+k}^{-1}(U^{(m)}_k)$.

\paragraph{Copula selection.} The choice of $C$ determines temporal dependence structure and path-event probabilities:

\begin{itemize}
\item \textbf{Parametric copulas.} Gaussian copula with correlation matrix $R$ captures linear persistence; Student-$t$ copula adds tail dependence, correcting underestimation of joint extremes that independence or Gaussian copulas produce. Estimate $R$ from historical data or domain knowledge.

\item \textbf{Vine copulas.} D-vines and C-vines decompose high-dimensional copulas into bivariate building blocks along lags \cite{aas_pair-copula_2009}, capturing asymmetric and nonlinear dependence.

\item \textbf{Empirical copulas.} Ensemble Copula Coupling (ECC) \cite{schefzik_uncertainty_2013}: sample independently from each $F_{T+k}$, then impose the rank structure of a reference ensemble to induce coherence. Variants include ECC-Q (quantile-based), ECC-R (random), ECC-T (time-based).
\end{itemize}

Critical assumption: The copula is not identified by marginals—different $C$ yield different path probabilities (Proposition~\ref{prop:nonident}). This conversion requires imposing temporal dependence structure from external sources. Independence copulas underestimate persistence and extremes. Tail-dependent copulas or vine structures mitigate this but require estimation from auxiliary data. ECC inherits biases from the reference ensemble. Time-varying dependence requires dynamic copulas or regime conditioning.

\subsection{From quantile forecasts}

Given quantile levels $\mathcal{Q} = \{q_\ell\}_{\ell=1}^L$ and values $\{Q_{T+k}(q_\ell)\}_{\ell=1}^L$ at each step, conversions require reconstruction of full marginal distributions.

\subsubsection{To parametric forecasts}
\label{sec:quant2dens}

A finite quantile grid constrains the CDF at $L$ points; infinitely many densities match these constraints. Uniqueness holds only under a fixed parametric family with sufficient informative levels (Proposition~\ref{prop:quantiles-to-parametric}). Three reconstruction methods \cite{arnold_first_2008,turnbull_empirical_1976,bahadur_note_1966,kiefer_bahadurs_1967}:

\paragraph{Moment matching.} For families with closed-form quantile-to-parameter maps (Gaussian, lognormal), invert directly. Example: Gaussian with two levels $q_1, q_2$ yields
\[
\mu = \frac{Q(q_2)z_{q_1} - Q(q_1)z_{q_2}}{z_{q_1} - z_{q_2}}, \quad \sigma = \frac{Q(q_2) - Q(q_1)}{z_{q_2} - z_{q_1}}
\]
where $z_q = \Phi^{-1}(q)$. For more complex families, use order statistics and maximum likelihood \cite{arnold_first_2008}.

\paragraph{Quantile regression.} Choose family $F(\cdot \mid \boldsymbol{\theta})$ and fit by minimizing quantile discrepancy \cite{klugman_loss_nodate,sgouropoulos_matching_2015,serfling_approximation_1980,a_w_van_der_vaart_asymptotic_1998}:
\[
\hat{\boldsymbol{\theta}} = \arg\min_{\boldsymbol{\theta}} \sum_{\ell=1}^L w_\ell (Q_{\boldsymbol{\theta}}(q_\ell) - Q_{T+k}(q_\ell))^2
\]
with weights $w_\ell \propto [q_\ell(1-q_\ell)]^{-1} f_{\boldsymbol{\theta}}(Q_{\boldsymbol{\theta}}(q_\ell))^2$. Iterate: initialize with equal weights, fit, update weights, refit.

\paragraph{Nonparametric interpolation with parametric tails.} Build a monotone interpolant $\widehat{Q}_{T+k}$ (piecewise linear or spline) from $\{(q_\ell, Q_{T+k}(q_\ell))\}$, invert to get $\widehat{F}_{T+k}$, then fit tail behavior beyond the quantile grid using a parametric family (\emph{e.g.}, GPD for upper tail). Fit by minimizing Cramér-von Mises distance \cite{hettmansperger_minimum_1994} or quantile distance \cite{devroye_non-uniform_2013}.

Caveats: Tail extrapolation requires parametric assumptions not contained in the quantile forecast. For risk tasks requiring extreme quantiles ($q < 0.01$ or $q > 0.99$) when training used $\{0.1,\ldots,0.9\}$, retraining is more reliable than extrapolation.

\subsubsection{To point forecasts}

If the median level $q=0.5$ is in the training set $\mathcal{Q}$, the median is directly available: $\hat{y}_{T+k} = Q_{T+k}(0.5)$; otherwise, interpolate from neighboring quantile levels. For the mean, integrate the quantile function via numerical quadrature. With trapezoidal rule:
\[
\widehat{\mathbb{E}}[Y_{T+k}] \approx \sum_{\ell=1}^{L-1} \frac{Q_{T+k}(q_\ell) + Q_{T+k}(q_{\ell+1})}{2}(q_{\ell+1} - q_\ell) + \text{tail corrections}.
\]
This uses the identity $\mathbb{E}[Y] = \int_0^1 Q(u) \, du$, approximated via the observed quantile grid.

Caveats: Coarse grids and unmodeled tails bias the mean estimate. Interpolation for the median is generally reliable, but mean estimation requires integrating over the full support including tails. All distributional information beyond the chosen statistic is lost.

\subsubsection{To trajectory ensembles}

Reconstruct marginals as in Section~\ref{sec:quant2dens}, then apply copula-based sampling as in Section~\ref{sec:dens2paths}.

Multiple assumptions required: Accuracy depends on both the marginal reconstruction quality and copula specification. Coarse quantile grids or misspecified tails propagate to path-level errors, particularly for first-passage and duration questions. Temporal dependence must be imposed from external sources.

\subsection{From point forecasts}
\label{sec:point_convert}

Point forecasts $\{y_{T+k}\}_{k=1}^h$ contain no native uncertainty information. Post-hoc methods add approximate uncertainty:

\paragraph{Split conformal prediction.} Partition data into training and calibration sets. On calibration data, compute residuals $r_i^{(k)} = |y_{t_i+k} - \hat{y}_{t_i+k}|$ for each lead $k$. For marginal intervals, set $q_k$ as the $(1-\alpha)$-quantile of $\{r_i^{(k)}\}$ and form $[\hat{y}_{T+k} - q_k, \hat{y}_{T+k} + q_k]$. For pathwise bands, use sup-norm scores $s_i = \max_k r_i^{(k)}/w_k$ where $w_k$ scales by lead-specific spread \cite{angelopoulos_gentle_2021,stankeviciute_conformal_2021}.

\paragraph{Residual bootstrap.} Estimate the residual distribution from historical forecast errors, stratified by time-of-day, day-of-week, or lead. Sample $e^{*(m)}_k$ from this distribution and form pseudo-trajectories $y^{(m)}_{T+k} = \hat{y}_{T+k} + e^{*(m)}_k$.

External assumptions required: Conformal methods guarantee coverage but not calibrated event probabilities—they answer "\textit{where will 95\% of realizations fall?}" not "\textit{what is $\Pr(\text{event})$?}". Bootstraps inherit the residual model and may miss regime changes or tail behavior. Both approaches impose uncertainty structure from external calibration data.

\subsection{Practical guidance}

\Cref{tab:conversion_summary} summarizes conversion methods, validity conditions, and key limitations.

\begin{table}
\centering
\small
\begin{tabular}{@{}lll@{}}
\toprule
\textbf{Conversion} & \textbf{Method} & \textbf{Notes} \\
\midrule
Trajectory → Quantiles & Empirical percentiles & §4.2.1: Direct; discards temporal dependence \\
Trajectory → Parametric & Fit family to empirical & §4.2.2: Direct; discards dependence; misspec. risk \\
Trajectory → Point & Mean/median & §4.2.3: Direct; discards all uncertainty \\
\midrule
Parametric → Quantiles & Invert CDF & §4.3.1: Direct (lateral); misspec. affects tails \\
Parametric → Point & Extract mean/median & §4.3.2: Direct; loses distribution \\
Parametric → Trajectory & Copula sampling & §4.3.3: Requires dependence assumptions \\
\midrule
Quantiles → Parametric & Moment/regression/interp. & §4.4.1: Requires distributional family; tail risk \\
Quantiles → Point & Median / integrate & §4.4.2: Approximation; grid coarseness \\
Quantiles → Trajectory & Reconstruct + copula & §4.4.3: Compounds multiple assumptions \\
\midrule
Point → Any & Conformal / bootstrap & §4.5: Requires external calibration data \\
\bottomrule
\end{tabular}
\vspace{0.75em}
\caption{Conversion methods and assumptions. Section references link to detailed procedures. Direct conversions compute target information from the source forecast without additional assumptions; upward conversions (toward trajectory ensembles) require structural assumptions (copulas, distributional families, calibration data) to reconstruct information not present in marginal forecasts.}
\label{tab:conversion_summary}
\end{table}

\paragraph{When to convert.} Marginalization (trajectory $\to$ simpler forms) is appropriate when the simpler form suffices for the intended task. If only marginal intervals are needed, converting trajectories to quantiles reduces storage and communication costs while preserving the relevant information. However, marginalization is irreversible—if path-dependent questions later arise, trajectories must be regenerated, which may not be possible without the original model.

\paragraph{When not to convert.} Conversions toward trajectory ensembles from marginal forecast types require imposing temporal dependence structure that is not identified by the source forecast. If dependence structure is unknown or dynamic, copula-based sampling can produce misleading path probabilities. In such cases, either (i) use a TSFM that natively outputs trajectories, or (ii) validate dependence assumptions against held-out data using path-aware metrics (Section~\ref{sec:evaluation}).

\paragraph{Validation requirements.} After any upward conversion, assess whether reconstructed outputs support the intended applications. For marginal-to-joint conversions, check path-dependent metrics (Energy Score, Variogram Score; Section~\ref{sec:evaluation}) on validation data. For quantile-to-parametric conversions, verify tail behavior if the application depends on extreme events. For point-to-probabilistic conversions, test calibration (PIT histograms, reliability diagrams) not just coverage.

\section{Evaluation}
\label{sec:evaluation}

Forecast evaluation must align with the applications forecasts are intended to support. For example, a trajectory ensemble with accurate marginals but poor temporal dependence will perform well on marginal metrics yet fail on path-dependent tasks like threshold crossing. Conversely, evaluating point forecasts using probabilistic metrics is meaningless without native uncertainty quantification. These examples illustrate a fundamental principle: the right metric depends on both the forecast type and the intended operational task.

We organize evaluation metrics by forecast type and operational task, emphasizing practical guidance for practitioners. \Cref{tab:metric_task_map} provides the complete mapping. The following subsections define each metric family, explain what they measure, and clarify when each is appropriate.

\begin{table}
\centering
\small
\begin{tabular}{@{}llll@{}}
\toprule
\textbf{Metric family} & \textbf{Forecast type} & \textbf{Tasks} & \textbf{Key metrics} \\
\midrule
Point accuracy & Point & None (baseline only) & MAE, MSE, MASE \\
Marginal calibration & Parametric, Quantiles & Pointwise intervals & Pinball, CRPS, WIS \\
Path dependence & Trajectory & Path-dependent tasks & Energy, Variogram \\
Event probabilities & Parametric, Trajectory & Events and crossing & Brier, IBS \\
Calibration diagnostics & Parametric, Quantiles, Trajectory & All probabilistic tasks & PIT, reliability, coverage \\
\bottomrule
\end{tabular}
\vspace{0.75em}
\caption{Metric families, applicable forecast types, and operational tasks from \Cref{sec:usecases}. Path-dependent tasks include pathwise bands, event probabilities, threshold crossing, and scenario generation.}
\label{tab:metric_task_map}
\end{table}

\subsection{Point forecast metrics}
Point forecasts are evaluated using deterministic error-based metrics including Mean Absolute Error (MAE), Mean Squared Error (MSE), and Mean Absolute Scaled Error (MASE) \cite{hyndman_another_2006}. These metrics measure central tendency accuracy—how close predicted values are to realized values—but convey no information about forecast uncertainty or the distribution of potential outcomes. While useful as baseline comparisons, they cannot evaluate any of the six operational tasks from \Cref{sec:usecases} because all of those tasks require explicit uncertainty quantification to assess risk, compute probabilities, or construct intervals.

Mathematically, for $h$ point forecast-realization pairs $(\hat{y}_t, y_t)$:
$$
\text{MAE} = \frac{1}{h}\sum_{t=1}^h |\hat{y}_t - y_t|, \quad
\text{MSE} = \frac{1}{h}\sum_{t=1}^h (\hat{y}_t - y_t)^2, \quad
\text{MASE} = \frac{\text{MAE}}{\frac{1}{n-1}\sum_{t=2}^n |y_t - y_{t-1}|}
$$
where MASE scales MAE by the in-sample one-step naive MAE computed on the training series of length $n$ \cite{hyndman_another_2006}.

\subsection{Marginal distributional metrics}

Proper scoring rules assess the quality of probabilistic forecasts by penalizing both bias and miscalibration \cite{gneiting_strictly_2007}. For quantile forecasts, \textbf{Pinball loss} measures the accuracy of predicted quantile levels:
\[
\rho_q(\hat{Q}_t, y_t) = (q - \mathbf{1}\{y_t < \hat{Q}_t(q)\})(\hat{Q}_t(q) - y_t)
\]
where $\hat{Q}_t(q)$ is the predicted $q$-quantile for observation $t$. The \textbf{Weighted Interval Score (WIS)} averages pinball losses across multiple quantile levels to provide an aggregate measure of interval quality \cite{bracher_evaluating_2021}. The \textbf{Continuous Ranked Probability Score (CRPS)} \citep{gneiting_strictly_2007} generalizes MAE to full predictive distributions:
$$
\text{CRPS}(\hat{F}_t, y_t) = \int_{-\infty}^\infty (\hat{F}_t(x) - \mathbf{1}\{y_t \le x\})^2 \, dx
$$
where $\hat{F}_t$ is the predictive CDF for observation $t$. It rewards both accuracy and sharpness, admits closed-form expressions for many common parametric families, and directly evaluates pointwise prediction intervals. \textbf{Log-likelihood} \citep{gneiting_strictly_2007} is another proper scoring rule but exhibits high sensitivity to tail misspecification, making it less robust for operational evaluation.

A critical limitation unites all marginal metrics: they evaluate per-step accuracy in isolation while completely ignoring temporal dependence structure. A forecast with perfect marginal calibration at each time step but incorrect temporal correlation will achieve excellent marginal scores yet fail on path-dependent tasks. Marginal metrics cannot distinguish trajectories with independent steps from those with strong autocorrelation, rendering them insufficient for pathwise bands, event probabilities, threshold crossing, and scenario generation. Per-step marginal scores cannot be aggregated to evaluate joint functionals without explicitly modeling temporal dependence.

\subsection{Multivariate and path-dependent metrics}

Path-dependent tasks require metrics that evaluate joint distributions over the forecast horizon, not just marginal accuracy at individual time steps. The \textbf{Energy Score} extends CRPS to multi-time (or multivariate) settings:
$$
\text{ES}\left(\hat{\mathbf{y}}^{(1)}, \ldots, \hat{\mathbf{y}}^{(M)}, \mathbf{y}\right) = \frac{1}{M}\sum_{m=1}^M \|\hat{\mathbf{y}}^{(m)} - \mathbf{y}\| - \frac{1}{2M^2}\sum_{m=1}^M\sum_{m'=1}^M \|\hat{\mathbf{y}}^{(m)} - \hat{\mathbf{y}}^{(m')}\|
$$
where $\|\cdot\|$ is the Euclidean norm, $\hat{\mathbf{y}}^{(m)}$ are forecast trajectories, and $\mathbf{y}$ is the realized path. It simultaneously penalizes both marginal forecast errors and incorrect temporal dependence structure \cite{gneiting_strictly_2007}, providing a proper scoring rule for trajectory ensembles that captures whether forecasted paths exhibit realistic autocorrelation, persistence, and cross-time relationships. The \textbf{Variogram Score} offers a related approach:
$$
\text{VS}(\hat{\mathbf{y}}^{(1)}, \ldots, \hat{\mathbf{y}}^{(M)}, \mathbf{y}) = \sum_{t<t'} w_{t,t'}\left[\left(|y_t - y_{t'}| - \frac{1}{M}\sum_{m=1}^M |\hat{y}^{(m)}_t - \hat{y}^{(m)}_{t'}|\right)^2\right]
$$
where $y_t$ denotes the $t$-th time step of the realized path $\mathbf{y}$, $\hat{y}^{(m)}_t$ denotes the $t$-th time step of forecast trajectory $m$, and weights $w_{t,t'}$ control which temporal lags receive emphasis \cite{scheuerer_variogram-based_2015}, enabling practitioners to prioritize short-range versus long-range dependence based on application requirements.

Both metrics are essential for validating copula-based conversions (\Cref{sec:dens2paths}) and for evaluating any forecast used in path-dependent tasks (pathwise bands, event probabilities, threshold crossing, and scenario generation). The diagnostic pattern is revealing: if a trajectory ensemble achieves low CRPS (good marginal calibration) but high Energy Score (poor dependence), then the marginal distributions are accurate but the temporal correlation structure is misspecified. Such a forecast is acceptable for pointwise prediction intervals where only marginals matter, but problematic for threshold crossing where persistence and run lengths depend critically on temporal dependence.

\subsection{Event-based and calibration metrics}

Event-based tasks (event probabilities and threshold crossing from \Cref{sec:use_case_event,sec:use_case_threshold}) require computing probabilities of specific events from TSFM forecasts. The methods for computing event probabilities differ by forecast type, as detailed in \Cref{sec:use_case_event,sec:use_case_threshold}: point forecasts cannot produce calibrated probabilities; parametric and quantile forecasts compute via marginal CDFs (requiring copulas for path-dependent events); trajectory ensembles use direct Monte Carlo estimation.

Once event probabilities are computed, the \textbf{Brier Score} evaluates their accuracy. For a validation set of $n$ forecast instances indexed by $i$, let $\hat{p}_i$ denote the predicted probability and $o_i \in \{0,1\}$ the observed binary outcome:
$$
\text{BS} = \frac{1}{n}\sum_{i=1}^n (\hat{p}_i - o_i)^2
$$

For threshold crossing, the \textbf{Integrated Brier Score (IBS)} evaluates time-to-event survival functions, where computation methods are detailed in \Cref{sec:use_case_threshold,app:survival}. For each validation instance $i$ with forecast horizon $h$, let $\hat{S}_i(k)$ denote the predicted probability that the hitting time exceeds $k$ steps ahead, and let $\tau_i$ denote the observed hitting time:
$$
\text{IBS} = \frac{1}{nh}\sum_{i=1}^n\sum_{k=1}^h (\hat{S}_i(k) - \mathbf{1}\{\tau_i > k\})^2
$$
Both metrics measure whether predicted event probabilities accurately reflect empirical frequencies—the operational requirement for event probabilities and threshold crossing.

Beyond evaluating individual forecasts, calibration diagnostics assess the fundamental reliability of a forecasting system. The \textbf{Probability Integral Transform (PIT)} evaluates marginal calibration. For validation instance $i$ with forecast origin at time $T_i$ and horizon step $k$:
\[
\text{PIT}_{i,k} = \hat{F}_{T_i+k}(y_{T_i+k})
\]
where $\hat{F}_{T_i+k}$ is the predictive CDF for time $T_i+k$ and $y_{T_i+k}$ is the realized value. Under perfect marginal calibration, $\{\text{PIT}_{i,k}\}_{i=1}^n$ should follow a uniform distribution for each fixed $k$ \cite{gneiting_probabilistic_2007}. Deviations reveal systematic biases: U-shaped PIT histograms indicate overconfidence (intervals too narrow), while inverse-U shapes indicate underconfidence (intervals too wide). \textbf{Reliability diagrams} provide complementary insight for event probabilities by binning predicted probabilities and comparing them to empirical frequencies, directly visualizing whether a forecast that claims ``30\% probability" actually realizes the event roughly 30\% of the time.

A critical distinction emerges when evaluating uncertainty quantification methods. Conformal prediction methods (\Cref{sec:convert}) provide finite-sample guarantees for \emph{coverage}—the fraction of realizations falling within prediction intervals—but do not guarantee \emph{calibration} of probability estimates. An interval construction may achieve exactly 95\% coverage while simultaneously assigning severely miscalibrated probabilities to tail events. For event-based tasks (event probabilities and threshold crossing), where operational choices depend on predicted probabilities rather than interval containment, calibration is necessary; coverage alone is insufficient. Among forecasts that achieve adequate calibration, practitioners should prefer sharper (more concentrated) predictions that provide greater operational value. However, optimizing for sharpness without first ensuring calibration produces overconfident forecasts that systematically mislead practitioners.

\subsection{Practical metric selection}

The preceding subsections establish a taxonomy of evaluation metrics organized by what they measure. Practitioners must select metrics aligned with their operational tasks. We provide concrete guidance for the six tasks from \Cref{sec:usecases}:

\begin{itemize}
\item \textbf{Pointwise intervals and window aggregates:} CRPS or pinball loss for distributional accuracy, supplemented with PIT histograms to verify marginal calibration. These tasks require only marginal distributions, making per-step metrics sufficient.

\item \textbf{Pathwise bands and scenario generation:} Energy Score or Variogram Score to evaluate joint distributions and temporal dependence. If formal coverage guarantees are needed for pathwise bands, apply split conformal calibration and report simultaneous coverage rates.

\item \textbf{Event probabilities and threshold crossing:} Brier Score for binary event probabilities, IBS for time-to-event distributions, and reliability diagrams to verify probability calibration. For multi-step events spanning multiple time points, add Energy Score to validate temporal dependence structure.
\end{itemize}

No single metric captures all dimensions of forecast quality—accuracy, calibration, sharpness, and temporal dependence interact in complex ways. Comprehensive evaluation requires reporting multiple metrics that target different aspects of performance. For example, a probabilistic forecast evaluation could report: (1) one marginal distributional metric (CRPS or pinball loss) quantifying per-step accuracy, (2) one calibration diagnostic (PIT histogram or empirical coverage rate) verifying reliability, and (3) if trajectory ensembles are available, one path-dependent metric (Energy Score or Variogram Score) assessing joint calibration. This minimal suite reveals critical trade-offs: a model might sacrifice marginal sharpness to achieve better joint calibration, a trade-off that benefits path-dependent tasks but not pointwise interval construction.

\section{Conclusion}

Time-series foundation models have achieved remarkable accuracy, yet accuracy alone does not determine practical value. The form of a forecast—point, quantile, parametric, or trajectory ensemble—fundamentally constrains which operational questions it can answer. Our analysis establishes when different forecast types suffice for specific operational tasks and when conversions between forecast types preserve or lose critical information.

Three core results emerge. First, trajectory ensembles are strictly most expressive, enabling direct conversion to all other forms through marginalization without additional assumptions. Conversely, reconstructing trajectory ensembles from marginals requires imposing temporal dependence structure via copulas or other external modeling. This asymmetry has practical implications: TSFMs that produce trajectory ensembles maximize application flexibility, while per-step forecast types lock users into marginal-only questions unless they invest in validated dependence modeling.

Second, many operational problems are inherently path-dependent. First-passage problems, threshold persistence, aggregate risk measures, and scenario generation all require joint predictive distributions over the forecast horizon. Our survey reveals a mismatch: the minority of published TSFMs natively produce trajectory ensembles while the majority produce point forecasts, yet path-dependent questions pervade operational forecasting across domains—from predictive maintenance to financial risk management to climate impact assessment.

Third, conversions between forecast types are not universally valid. Our impossibility results (Propositions 1-3) formalize when marginals cannot determine path-event probabilities without specifying temporal dependence structure. Practitioners who convert marginal forecast types to trajectory ensembles via independence assumptions or ad-hoc copulas may obtain forecasts that appear reasonable on marginal metrics (CRPS) yet fail catastrophically on path-aware metrics (Energy Score) and real applications.

For TSFM developers, trajectory ensembles maximize downstream flexibility when computational resources permit. When efficiency demands simpler forecast types, quantile forecasts serve interval-based use cases while parametric forecasts enable analytical shortcuts for specific distributional families. For practitioners selecting TSFMs, our task-forecast mapping (\Cref{tab:use_cases}) provides concrete guidance. Evaluation must align with intended use, as detailed in \Cref{sec:evaluation}.

The question ``\textit{which TSFM is best?}" is incomplete without specifying ``\textit{best for what application?}" A point forecaster with state-of-the-art MAE cannot answer probability questions. A marginal forecaster with excellent CRPS cannot reliably estimate first-passage times without validated dependence modeling. A trajectory ensemble with poor calibration provides misleading risk estimates despite outputting the richest form. Progress in time-series foundation models must move beyond accuracy improvements toward forecast types and evaluation metrics that serve operational needs. When multiple models achieve similar marginal performance, forecast type becomes the primary differentiator of value. The ultimate measure of a TSFM is not its benchmark score but its ability to inform the applications it was built to serve.

%
%
\bibliographystyle{plainnat}
\bibliography{references}

\newpage
\appendix

\section{Survey of TSFM Forecast Types}
\label{app:survey}

We surveyed over 50 published TSFMs to understand the distribution of forecast types in current practice. \Cref{tab:tsfm_survey} classifies each model by forecast type as originally published. We acknowledge that many architectures could be extended to produce richer forecast types; our classification reflects the authors' proposed implementation. The temporal boundary of TSFMs is imprecise—the term gained prominence around 2022 following advances in large language models—but earlier work (DeepAR, MQ-RNN) pioneered similar foundation model capabilities for time series.

We screened TSFM papers from 2017--October 2025 across arXiv and major ML venues (NeurIPS, ICML, ICLR, KDD, AAAI, \emph{etc.}), retaining models with general-purpose time-series forecasting capability. Each paper is assigned to its \emph{richest native output} (point, quantile, parametric per-step marginals, or trajectory ensemble). Multi-output models are counted once under the richest native output. Dates in Table~\ref{tab:tsfm_survey} reflect the earliest arXiv revision at collection time.

\Cref{fig:papers_by_output_type} shows the evolution of forecast types over time. Notable trends include: (1) persistent dominance of point forecasts across all years, (2) recent growth in trajectory sampling methods (Chronos and variants), (3) emergence of repurposed language models, which predominantly produce point forecasts, and (4) relatively sparse development of quantile-based methods.

\begin{table}[!h]
\centering
\scriptsize
\setlength{\tabcolsep}{3pt}
\renewcommand{\arraystretch}{1.1}

\begin{tabular}{@{}%
  L{1.35cm}L{2.6cm}L{1.4cm}L{1.1cm} @{\hspace{8pt}}
  L{1.35cm}L{2.6cm}L{1.4cm}L{1.1cm}@{}}
\toprule
\textbf{DOI/arXiv} & \textbf{Algorithm Name} & \textbf{Architecture} & \textbf{Type} &
\textbf{DOI/arXiv} & \textbf{Algorithm Name} & \textbf{Architecture} & \textbf{Type} \\
\cmidrule(r){1-4}\cmidrule(l){5-8}

1711.11053 & MQ-RNN, MQ-CNN \cite{wen_multi-horizon_2018} & Non-trans. & Quantiles &
2402.07570 & GTT \cite{feng_only_2024} & Transformer & Point \\
1704.04110 & DeepAR \cite{salinas_deepar_2020} & Non-trans. & Parametric &
2403.00131 & UniTS \cite{gao_units_2024} & Transformer & Point \\
2106.10370 & MLE(ZNBP) \cite{awasthi_benefits_2021} & Non-trans. & Parametric &
2403.01742 & DIFFUSION-TS \cite{yuan_diffusion-ts_2024} & Transformer & Trajectory \\
2107.03502 & CSDI \cite{tashiro_csdi_2021} & Transformer & Trajectory &
2403.05713 & tsGT \cite{kucinski_tsgt_2024} & Transformer & Trajectory \\
2205.13504 & LTSF-Linear \cite{zeng_are_2022} & Non-trans. & Point &
2403.07815 & Chronos \cite{ansari_chronos_2024} & Transformer & Trajectory \\
2205.13158 & SwinVRNN \cite{hu_swinvrnn_2023} & Transformer & Trajectory &
2407.07874 & Toto \cite{cohen_toto_2024} & Transformer & Parametric \\
2210.02186 & TimesNet \cite{wu_timesnet_2023} & Non-trans. & Point &
2409.02322 & TimeDiT \cite{cao_timedit_2025} & Transformer & Trajectory \\
2211.14730 & PatchTST \cite{nie_time_2022} & Transformer & Point &
2409.15367 & Wass.+Chronos \cite{chernov_fine-tuning_2024} & Transformer & Trajectory \\
2302.11939 & FPT \cite{zhou_one_2023} & Repurp. LM & Point &
2409.16040 & Time-MoE \cite{shi_time-moe_2025} & Transformer & Point \\
2304.08424 & TiDE \cite{das_long-term_2023} & Non-trans. & Point &
2410.03806 & MetaTST \cite{dong_metadata_2024} & Transformer & Point \\
2306.09364 & TSMixer \cite{ekambaram_tsmixer_2023} & Non-trans. & Point &
2410.04803 & Timer-XL \cite{liu_timer-xl_2025} & Transformer & Point \\
2308.08241 & TEST \cite{sun_test_2023} & Repurp. LM & Point &
2410.09385 & Mamba4Cast \cite{bhethanabhotla_mamba4cast_2024} & Non-trans. & Point \\
2308.08469 & LLM4TS \cite{chang_llm4ts_2023} & Repurp. LM & Point &
2410.10469 & Moirai-MoE \cite{liu_moirai-moe_2024} & Transformer & Parametric \\
2310.10688 & TimesFM \cite{das_decoder-only_2023} & Transformer & Point &
2410.22269 & Fourier+Chronos \cite{gillman_fourier_2024} & Transformer & Trajectory \\
2310.10688 & TimesFM-Q \cite{das_timesfm_2023} & Transformer & Quantiles &
2410.24087 & TimesFM-ICF \cite{das_-context_2024} & Transformer & Point \\
2310.01327 & TACTiS-2 \cite{ashok_tactis-2_2023} & Transformer & Parametric &
— & ChronosBolt \cite{ansari_chronosbolt_2024} & Transformer & Quantiles \\
2310.01728 & Time-LLM \cite{jin_time-llm_2024} & Repurp. LM & Point &
2411.02941 & TSMamba \cite{ma_mamba_2024} & Non-trans. & Point \\
2310.03589 & TimeGPT \cite{garza_timegpt-1_2023} & Transformer & Point &
2411.08249 & RAF \cite{tire_retrieval_2024} & Transformer & Trajectory \\
2310.04948 & TEMPO \cite{cao_tempo_2023} & Repurp. LM & Point &
2501.02945 & TabPFN-T \cite{hoo_tables_2025} & Non-trans. & Point \\
2310.06625 & iTransformer \cite{liu_itransformer_2023} & Transformer & Point &
2502.00816 & Sundial \cite{liu_sundial_2025} & Transformer & Trajectory \\
2310.07820 & LLMTime \cite{gruver_large_2023} & Repurp. LM & Trajectory &
2503.04118 & TimeFound \cite{xiao_timefound_2025} & Transformer & Point \\
2310.08278 & Lag-Llama \cite{rasul_lag-llama_2024} & Transformer & Parametric &
2505.13033 & TSPulse \cite{ekambaram_tspulse_2025} & Non-trans. & Point \\
2310.09751 & UniTime \cite{liu_unitime_2024} & Transformer & Point &
2505.23719 & TiRex \cite{auer_tirex_2025} & Non-trans. & Quantiles \\
2311.01933 & ForecastPFN \cite{dooley_forecastpfn_2023} & Transformer & Point &
2506.03128 & COSMIC \cite{auer_zero-shot_2025} & Transformer & Quantiles \\
2401.03955 & TTM \cite{ekambaram_tiny_2024} & Non-trans. & Point &
2508.19609 & FinCast \cite{zhu_fincast_2025} & Transformer & Quantiles \\
2402.02592 & Moirai \cite{woo_unified_2024} & Transformer & Parametric &
2510.15821 & Chronos-2 \cite{ansari_chronos-2_2025} & Transformer & Quantiles \\
2402.03885 & MOMENT \cite{goswami_moment_2024} & Transformer & Point & & & & \\
\bottomrule
\end{tabular}

\caption{Survey of TSFM literature classified by forecast type. Forecast types: Point, Parametric, Quantiles, Trajectory. Architecture: Transformer, Non-transformer, Repurposed LM (language model).}
\label{tab:tsfm_survey}
\end{table}

\begin{figure}[!h]
  \centering
  \includegraphics[width=.6\textwidth]{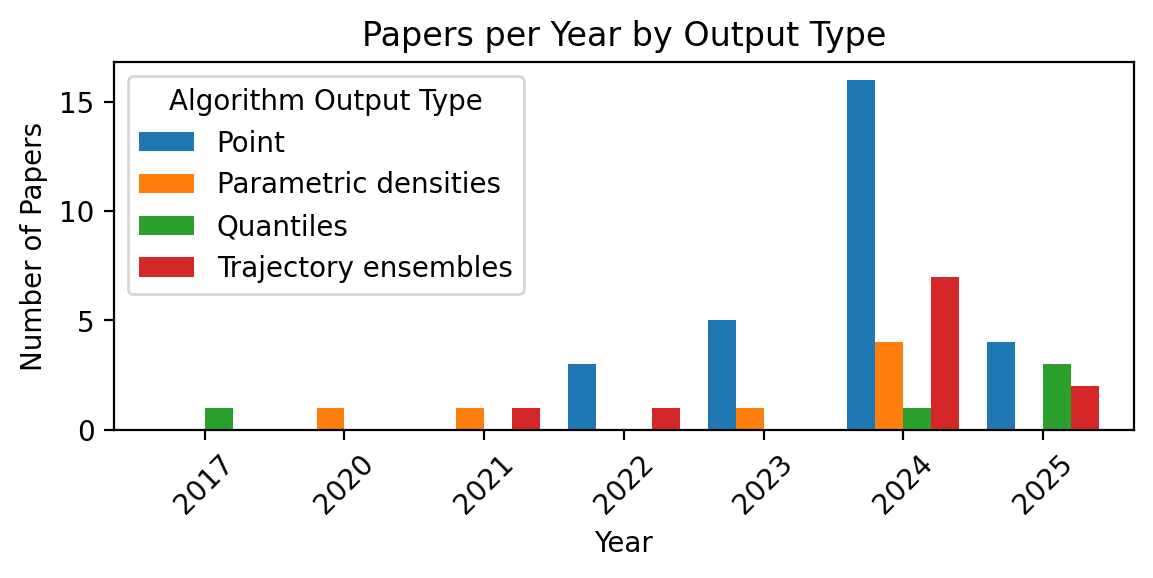}
  \caption{TSFM publications by forecast type over time. Point forecasts dominate across all years; trajectory ensembles show recent growth but remain a minority.}
  \label{fig:papers_by_output_type}
\end{figure}

\section{Extended Mathematical Derivations}
\label{app:b}

This appendix provides detailed formulations for selected operational tasks from \Cref{sec:usecases}.

\subsection{Value-at-Risk: Complete Formulation}
\label{app:var}

This section expands on the brief VaR discussion in \Cref{sec:use_case_event}.

Let $y_t$ denote the portfolio (log-)return at time $t$. For a horizon of $h$ trading days (\emph{e.g.}, $h=5$ for one week), define the aggregated return:
\[
R_{T+1:T+h} := \sum_{k=1}^{h} y_{T+k}
\]
and loss:
\[
L_{T+1:T+h} := -R_{T+1:T+h}.
\]

The $(100\times\alpha)\%$ Value-at-Risk over horizon $h$ is the $(1-\alpha)$-quantile of the predictive loss distribution:
\[
\mathrm{VaR}_{\alpha,h}(\mathbf{y}_{1:T}) := \inf\left\{x: \Pr\big(L_{T+1:T+h} \le x \mid \mathbf{y}_{1:T}\big) \ge 1-\alpha\right\}.
\]

Equivalently, $\Pr\big(R_{T+1:T+h} < -\mathrm{VaR}_{\alpha,h} \mid \mathbf{y}_{1:T}\big) = \alpha$. If one instead fixes a threshold (\emph{e.g.}, a 5\% weekly loss), the tail event probability is $\Pr\big(L_{T+1:T+h} > 0.05 \mid \mathbf{y}_{1:T}\big)$. VaR is the inverse of this tail probability as a function of $\alpha$.

\paragraph{Computation by forecast type.}
\begin{itemize}
  \item \textbf{Point forecasts:} Not suitable (no native uncertainty quantification).
  \item \textbf{Parametric forecasts:} Compute $\mathrm{VaR}_{\alpha,h}$ via the inverse CDF of the loss distribution implied by $\{\boldsymbol{\theta}_{T+k}\}_{k=1}^h$. The quality depends critically on: (1) correct specification of marginal distributions, particularly tail behavior, and (2) the assumed temporal dependence structure (copula or independence). Misspecified tails or dependence severely bias tail risk estimates.
  \item \textbf{Quantiles:} If the model produces quantiles at level $\alpha$, read off $\mathrm{VaR}_{\alpha,h}$ directly from the predicted loss quantile. For levels $\alpha \notin \mathcal{Q}$, interpolate within the training grid. However, TSFMs typically train on $\mathcal{Q} = \{0.1, \ldots, 0.9\}$; reliable prediction of extreme quantiles outside this range (\emph{e.g.}, $\alpha=0.01$ for 99\% VaR) generally requires retraining with tail-focused quantile levels.
  \item \textbf{Trajectory ensembles:} Estimate $\mathrm{VaR}_{\alpha,h}$ as the empirical $\alpha$-quantile of $\{L^{(m)}_{T+1:T+h}\}_{m=1}^M$ where $L^{(m)} = -\sum_{k=1}^h y^{(m)}_{T+k}$. For tail event probabilities, estimate $\Pr(L > \ell)$ by the fraction of simulated paths exceeding threshold $\ell$. Extreme quantiles require large sample sizes $M$ for stable estimation.
\end{itemize}

See \cite{jorion_value_2006,mcneil_quantitative_2015,diebold_comparing_1995} for further background on VaR estimation and evaluation.

\subsection{Survival Analysis for Threshold Crossing}
\label{app:survival}

This section expands on the brief threshold crossing discussion in \Cref{sec:use_case_threshold}.

Given threshold $C \in \mathbb{R}$, horizon $h$, and crossing direction $\triangleright \in \{\ge, \le\}$, define the first hitting time:
\[
\tau = \inf\{k \in \{1,\dots,h\}: y_{T+k} \triangleright C\}
\]
with right-censoring if no crossing occurs by $h$: $\tau > h$.

The goal is the predictive distribution $p(\tau \mid \mathbf{y}_{1:T})$, characterized by:

\paragraph{Survival function:}
\[
S(k) = \Pr(\tau > k \mid \mathbf{y}_{1:T})
\]

\paragraph{Hazard function:}
\[
h_k = \Pr(\tau = k \mid \tau \ge k, \mathbf{y}_{1:T})
\]

These satisfy the relationship:
\[
S(k) = \prod_{j=1}^{k}(1-h_j) \quad \text{and} \quad \Pr(\tau=k) = S(k-1)h_k.
\]

\paragraph{Computation from marginal CDFs.}
Let $F_{T+k}(u) = \Pr(y_{T+k} \le u \mid \mathbf{y}_{1:T})$ denote the predictive CDF at step $k$. The per-step crossing probability is:
\[
p_k := \Pr(y_{T+k} \triangleright C) = \begin{cases} 
1-F_{T+k}(C) & \text{if } \triangleright = \ge\\[2pt] 
F_{T+k}(C) & \text{if } \triangleright = \le
\end{cases}
\]

Under the independence approximation (crossing at step $k$ is independent of crossing at other steps):
\[
S(k) = \prod_{j=1}^{k}(1-p_j), \quad \Pr(\tau=k) = S(k-1)p_k, \quad h_k = p_k.
\]

This assumes a Markov-like structure where crossing probability at step $k$ depends only on the marginal distribution at that step, not on the full path history. This is often unrealistic for persistent processes—autocorrelation causes runs of high/low values, making actual crossing probabilities path-dependent. The independence approximation typically underestimates persistence and run lengths.

\paragraph{Answering specific example questions from \Cref{sec:use_case_threshold}.}

\textit{Vibration safety (probability of exceeding threshold over time):} Set $C$ to the vibration safety limit and $\triangleright = \ge$. The requested "probability of crossing as a function of time" is the cumulative distribution function:
\[
F_\tau(k) = \Pr(\tau \le k) = 1-S(k).
\]
An ensemble forecast yields the empirical $F_\tau(k)$ directly.

\textit{Air pollution persistence (staying above limit for $>3$ consecutive days):} Compute:
\[
\Pr\!\big(y_{T+1}\ge C,\dots,y_{T+4}\ge C\big).
\]
With trajectory ensemble: $\tfrac{1}{M}\sum_m \mathbf{1}\{\min_{j=1}^4 y^{(m)}_{T+j}\ge C\}$. Using the independence approximation: $\prod_{j=1}^4 p_j$ where $p_j=\Pr(y_{T+j}\ge C)$.

For further background on discrete-time survival modeling, see \cite{suresh_survival_2022}. On first-passage processes, see \cite{redner_guide_2001}. Additional survival analysis methods are discussed in \cite{fu_survival_2023}.

\subsection{Insurance Loss Scenario Ranking}
\label{app:insurance}

This section expands on the brief insurance scenario ranking discussion in \Cref{sec:use_case_scenarios}.

Let $y_t$ be a hazard intensity index (\emph{e.g.}, daily maximum wind speed, rainfall, temperature anomaly). Let $g: \mathbb{R}^h \to \mathbb{R}_+$ be a loss function mapping weather trajectories to financial losses. The loss over horizon $h$ is:
\[
L = g(\mathbf{y}_{T+1:T+h}) \ge 0.
\]

The goal is to generate scenarios, compute losses, and rank by probability and severity.

\paragraph{Procedure:}
\begin{enumerate}
\item \textbf{Sample trajectories:} Draw $\{\mathbf{y}^{(m)}_{T+1:T+h}\}_{m=1}^M \sim p(\mathbf{y}_{T+1:T+h}\mid \mathbf{y}_{1:T})$.

\item \textbf{Map to losses:} Compute $L^{(m)}=g(\mathbf{y}^{(m)})$ for each trajectory $m$.

\item \textbf{Empirical loss distribution:} 
\[
\widehat{F}_L(\ell)=\tfrac{1}{M}\sum_{m=1}^M \mathbf{1}\{L^{(m)}\le \ell\}.
\]

\item \textbf{Scenario probabilities:} 
\begin{itemize}
\item For i.i.d. draws: assign uniform weights $w_m=\tfrac{1}{M}$.
\item For structured scenarios: cluster paths using k-medoids or hierarchical clustering on features like peak intensity, duration, and spatial footprint. Set cluster weights $w_k=\frac{\#\text{paths in cluster }k}{M}$.
\end{itemize}

\item \textbf{Severity metrics:}
\begin{itemize}
\item Per scenario: severity $s_m=L^{(m)}$.
\item Per cluster: use mean loss $\bar{L}_k$ or a tail quantile (\emph{e.g.}, 95th percentile loss within cluster).
\end{itemize}

\item \textbf{Ranking:}
\begin{itemize}
\item By probability-weighted severity: rank scenarios by $r_m = w_m \cdot s_m$.
\item By exceedance probability: produce the loss exceedance curve:
\[
\widehat{\mathrm{EP}}(x) = 1-\widehat{F}_L(x) = \Pr(L > x).
\]
\item By conditional severity: rank clusters by $\bar{L}_k$ or tail quantile, then weight by $w_k$.
\end{itemize}
\end{enumerate}

This framework supports reserve setting (capital needed to cover $\alpha$-quantile loss), premium pricing (expected loss plus risk margin), and stress testing (identify high-severity low-probability scenarios).

For theoretical foundations of scenario-based planning, see \cite{murphy_value_1977,arrow_optimal_1951}.

\end{document}